\documentclass{article}

\usepackage[preprint]{neurips_2019}

\usepackage{wrapfig,graphicx,lipsum}

\usepackage[utf8]{inputenc} 
\usepackage[T1]{fontenc}    
\usepackage{hyperref}       
\usepackage{url}            
\usepackage{booktabs}       
\usepackage{amsfonts}       
\usepackage{nicefrac}       
\usepackage{microtype}      

\usepackage{appendix}

\usepackage{algorithm}
\usepackage[noend]{algpseudocode}

\usepackage{commath}
\usepackage{amsmath}
\usepackage{color}
\usepackage{epsfig}
\usepackage{graphicx}
\usepackage{amssymb}
\usepackage{mathtools}
\usepackage{comment}
\usepackage{caption}
\usepackage{subcaption}
\usepackage{booktabs}       
\usepackage{amsthm}
\usepackage{microtype}
\usepackage{graphicx}

\theoremstyle{definition}
\newtheorem{definition}{Definition}[section]

\usepackage[utf8]{inputenc}
\usepackage[english]{babel}
\newtheorem{theorem}{Theorem}[section]
\newtheorem{lemma}[theorem]{Lemma}

\newcommand{\R}{\mathbb{R}}
\newcommand{\G}{\mathcal{G}}
\newcommand{\T}{\mathcal{T}}
\newcommand{\E}{\mathrm{\textbf{E}}}
\newcommand{\Var}{\mathrm{\textbf{Var}}}
\newcommand{\Cov}{\mathrm{\textbf{Cov}}}
\newcommand{\etal}{\textit{et al}. }

\newcommand{\csdim}{c}

\newcommand{\msdim}{m}

\newcommand\CS{\mathrm{CS}}

\newcommand\HCS{\mathrm{HCS}}
\newcommand\CTS{\mathrm{CTS}}
\newcommand\IFFT{\mathrm{IFFT}}
\newcommand\FFT{\mathrm{FFT}}

\title{Higher-order Count Sketch: Dimensionality Reduction That Retains Efficient Tensor Operations}

\author{%
  Yang Shi\\
  University of California, Irvine\\
  Irvine, CA 92617 \\
  \texttt{shiy4@uci.edu} \\
   \And
   Animashree Anandkumar \\
   California Institute of Technology \\
   Pasadena, CA, 91101 \\
   \texttt{anima@caltech.edu} \\
}

\begin{document}

\maketitle


\begin{abstract}
Sketching is a randomized dimensionality-reduction method that aims to preserve relevant information in large-scale datasets. Count sketch is a simple popular sketch which uses a randomized hash function to achieve compression. 
In this paper, we propose a novel extension known as  Higher-order Count Sketch (HCS). 
While count sketch uses a single hash function, HCS uses multiple (smaller) hash functions for sketching.   HCS reshapes the input (vector) data into a higher-order tensor and employs a tensor product of the random hash functions to compute the sketch. 
This results in an exponential saving (with respect to the order of the tensor) in the memory requirements of the  hash functions, under certain conditions on the input data. Furthermore, when the input data itself has an underlying structure in the form of various tensor representations such as the Tucker decomposition, we obtain significant advantages. We derive efficient (approximate) computation of various tensor operations such as tensor products and tensor contractions directly on the sketched data. Thus, HCS is  the first sketch to fully exploit the multi-dimensional nature of higher-order tensors.
We apply HCS to tensorized neural networks where we replace fully connected layers with sketched tensor operations. We achieve nearly state of the art accuracy with significant compression on the image classification benchmark.
\end{abstract}

\section{Introduction}



Modern machine learning involves processing of large-scale datasets.  Dimensionality-reduction methods attempt to compress the input data while preserving relevant information. Sketching is a popular class of such techniques which aims to reduce memory and computational requirements by using simple randomized hash functions that map the input data to a reduced-sized output space.
Count Sketch (CS)~\citep{cs} is a simple sketch that has been applied in many settings such as estimation of internet packet streams~\citep{internet} and tracking most frequent items in a database~\citep{Cormode:2005}. It uses a simple data-independent random hash function and random signs to combine the input data elements. Despite its simplicity, it enjoys many desirable properties such as unbiased estimation and the ability to approximately perform certain operations directly on the low-dimensional sketched space, e.g., vector inner products and outer products.  However, CS is memory inefficient when the data is large. The bottleneck is that it needs to generate a hash table as large as the data size.



Another drawback of CS is that it assumes vector-valued data and does not exploit further structure  if data is multi-dimensional. But many modern machine learning and data mining applications involve manipulating large-scale multi-dimensional data. For instance, data can be multi-modal or multi-relational (e.g., a combination of image and text), and intermediate computations can involve higher-order tensor operations (e.g., layers in a tensorized neural network).  Memory, bandwidth, and computational requirements are usually bottlenecks when these operations are done at scale. Efficient dimensionality reduction schemes that exploit tensor structures can significantly alleviate this issue if they can find a compact representation while preserving accuracy.



\textbf{Main contributions: }We extend count sketch to Higher-order Count Sketch (HCS), which is the first sketch to fully exploit the multi-dimensional nature of higher-order tensors.  It reshapes the input (vector) data to a higher-order tensor of a fixed order. It utilizes multiple randomized hash functions: one for each mode of the tensor. 
The mapping in HCS is obtained by the tensor product of these hash functions. Figure~\ref{fig:overall} demonstrates the process. 
We show a memory saving in storing the hash map:  if the input data size is $O(d)$ and HCS uses $l$-th order tensor for sketching, we reduce the hash memory requirements from $O(d)$ to $O(l\sqrt[l]{d})$, compared to the count sketch, under certain conditions.  
 
 

The conditions for obtaining the best-case memory savings from HCS are related to the concentration of input entries with large magnitudes and require these large entries to be sufficiently spread out. Intuitively, this is because the hash indices in HCS are correlated and we cannot have all the input to be clustered together. If we are allowed  multiple passes over the input data 
, a simple (in-place) reshuffle  to spread out the large entries will fix this issue, and thus allows us to obtain maximum memory savings in storing hash functions. 
 

When the input data has further structure as a higher-order tensor, HCS is able to exploit it.  HCS allows for efficient (approximate) computations of tensor operations such as tensor products and tensor contractions by directly applying these operations on the sketched components. 
We obtain exponential saving with respect to the order of the tensor in the memory requirements for tensor product and contraction when compared to sketching using count sketch.
We also show $O(r^{N-1})$ times improvement in computation and memory efficiency for computing a $Nth$-order rank-$r$ Tucker tensor when compared to applying CS to each rank-1 component.  
The computation and memory improvement over CS of these operations are shown in Table~\ref{table:kron-summary}. 

We compare HCS and CS for tensor product and tensor contraction compression using synthetic data. HCS outperforms CS in terms of computation efficiency and memory usage: it uses $200 \times$ less compression time  and $40 \times $ less memory while keeping the same recovery error, compared to CS. Besides, we apply HCS for approximating tensor operations in tensorized neural networks. These networks replace fully connected layers with multi-linear tensor algebraic operations. 
Applying HCS to tensor operations results in further compression while preserving accuracy. 
We obtain $90\%$ test accuracy on CIFAR10 dataset with $80\%$ memory savings  on the last fully connected layer, compared to the baseline ResNet18. 

\begin{figure}
\begin{minipage}{0.45\textwidth}
  \centering
    \includegraphics[width = 4.8cm]{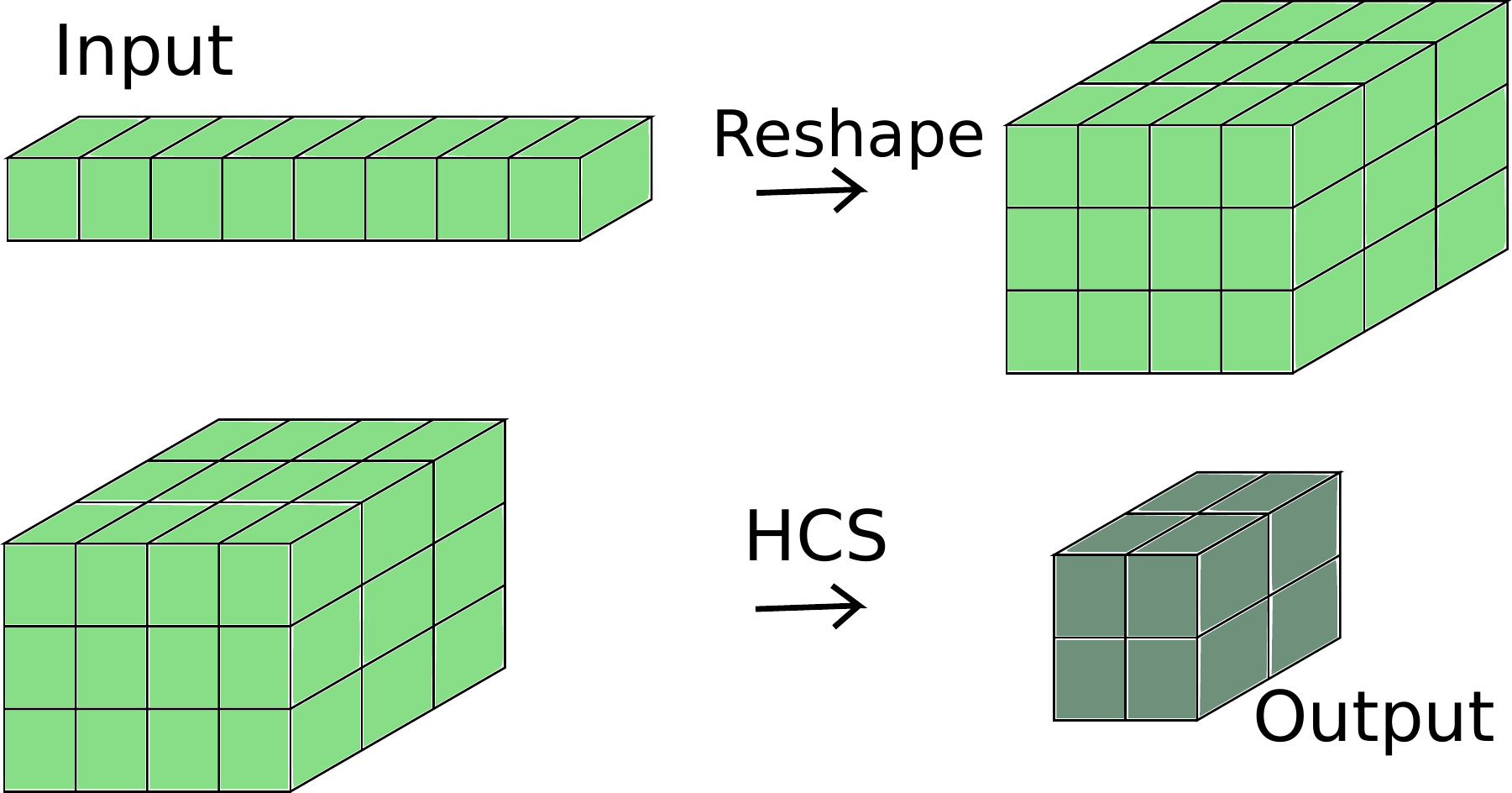}
  \captionof{figure}{Higher-order count   sketch reshapes input vector into higher-order tensor and sketches it into a (smaller) tensor of the same order.}
\label{fig:overall}
\end{minipage}
\hfill
\begin{minipage}{0.5\textwidth}
  \centering
  \includegraphics[width= 7cm]{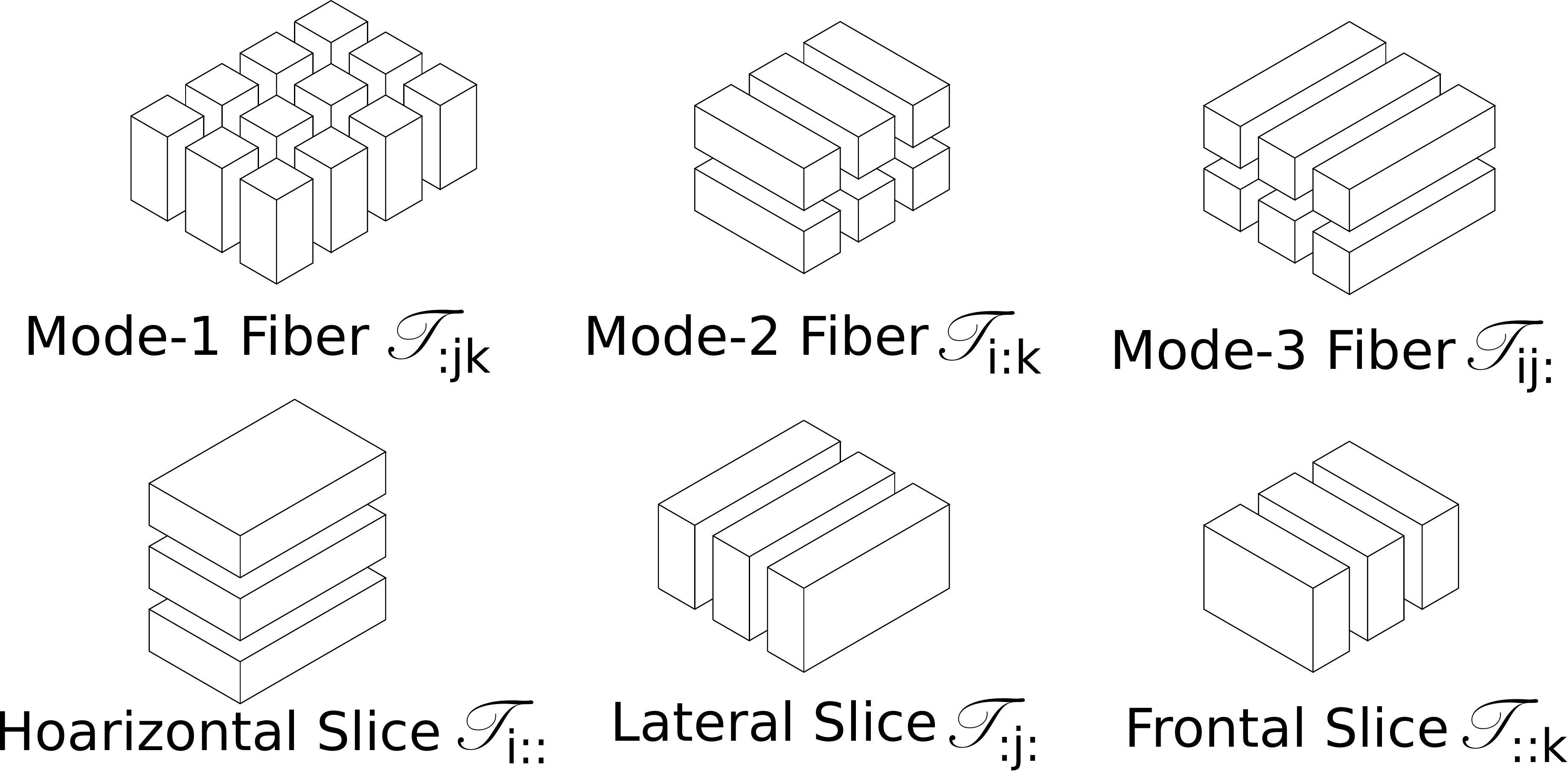}
\captionof{figure}{Fibers and slices of a third-order tensor.}
\label{fig:tensor-norm}
\end{minipage}
\end{figure}

\textbf{Related work:}
Singular value decomposition (SVD) is perhaps the most popular dimensionality reduction technique~\citep{svd}. However, when data is not inherently low rank or has other constraints such as  sparsity and non-negativity, SVD is not suitable. Other matrix decomposition techniques try to impose more structure on matrices~\citep{cur,cur-tensor,rpca}.



In contrast to matrix techniques which make stringent assumptions on underlying structure, sketching is designed for compressing vector-valued data with almost no assumptions~\citep{Bringmann2017,sketch,featurehash}. Count Sketch (CS)~\citep{cs} was proposed to estimate the frequency of each element in a stream.~\citet{cmm} propose a fast algorithm to compute CS of an outer product of two vectors using FFT properties. They prove that the CS of the outer product is equal to the convolutions between the CS of each input. This allows for vector operations such as inner product and outer product to be directly computed in the sketch space.
Since then, many variants of count sketch have been proposed that preserve different properties of underlying data. Min-hash~\citep{one-bit} is a technique for estimating how similar two sets of data are. An extension of that is one-bit CP-hash~\citep{cphash} which generates concurrent hash table for multi-core processors. To make use of parallel computing resources, 2-of-3 cuckoo hashing~\citep{Amossen} is proposed based on cuckoo hashing~\citep{cuckoo}.


Sketching can also be applied to multi-dimensional data. Tensor sketch~\citep{ts-k} is proposed to approximate  non-linear kernels. 
It has been applied to approximately compute tensor CP decomposition~\citep{ts,paraS} and Tucker decomposition~\citep{tucker-ts}.  
~\cite{CBP} introduce compact bilinear pooling to estimate joint features from different sources.  In Visual Question Answering task, people use compact bilinear pooling to compute joint features from language and image~\citep{MCB}. However, all these sketching techniques are sketching tensors into a vector, which destroys their multi-dimensional structure. This does not make it possible to do tensor operations efficiently in the sketched space.

In addition to sketching, efficient multi-dimensional data operation primitives can boost the computation performance.
A Low-overhead interface is proposed for multiple small matrix multiplications on NVIDIA GPUs~\citep{gemm}. 
\cite{sparsetensormatrix, blas} optimize tensor matrix contraction on GPUs 
by avoiding data transformation.
High-Performance Tensor Transposition~\citep{hptt} is one of the open-source library that performs efficient tensor contractions. In future, we can leverage these advances to further speed up tensor sketching operations. 

\textbf{Important tensor applications:} We focus on tensor sketching because  data is inherently multi-dimensional in many settings.
In probabilistic model analysis, tensor decomposition is the crux of model estimation via the method of moments. A variety of models such as topic models, hidden Markov models, Gaussian mixtures etc.,   can be efficiently solved via the tensor decomposition techniques under certain mild assumptions~\citep{animajmlr}.  ~\cite{basketball} analyze spatio-temporal basketball data via tensor decomposition.
Tensor methods are also relevant in deep learning.
\cite{rose-rnn} learn the nonlinear dynamics in recurrent neural networks directly using higher-order state transition functions through tensor train decomposition.
\cite{jean} propose tensor contraction and  regression layers in deep convolutional neural networks. 
\section{Preliminaries}

\textbf{Count Sketch}
\label{sec:pre-cs}
Count Sketch(CS)~\citep{cs} was first proposed to estimate most frequent data value in a streaming data.
\begin{definition}[Count Sketch]{}
Given two 2-wise independent random hash functions h:$[n]\to[c]$ and s:$[n] \to \{\pm 1\}$. Count Sketch of a vector $u \in \R^n$ is denoted by $CS(u) = \{CS(u)_1, \cdots ,CS(u)_c\} \in \R^{c}$ where $CS(u)_j := \sum_{h(i)=j} s(i)u_i$.
\end{definition}
In matrix format, we can write it as $CS(u) = H(s \circ u)$, where $H \in \R^{c \times n}$, $H(j,i) = 1$, if $h(i)=j$, for $\forall i \in [n]$,  otherwise $H(j,i) = 0$, $\circ$ is the sign for element-wise product.
The estimation can be made more robust by taking $b$ independent sketches of the input and calculating the median of the $b$ estimators.
\cite{cs} prove that the CS is an unbiased estimator with variance bounded by the 2-norm of the input. See Appendix~\ref{app:csms} for detailed proof. \cite{cmm} use CS and propose a fast algorithm to compute count sketch of an outer product of two vectors. 
\begin{equation}
\CS(u \otimes v) = \CS(u) * \CS(v)
\label{eqn:cs-outerproduct}
\end{equation}
 The convolution operation (represented using $*$) can be transferred to element-wise product using FFT properties. Thus, the computation complexity reduces from $O(n^2)$ to $O(n+c\log c)$, if the vectors are of size $n$ and the sketching size is $c$.

Some notations we use in the following paper are:  $\hat{u}$: decompression of $u$, $[n]$: set of $\{1,2,\dots, n\}$.
\section{Higher-order count sketch on vector-valued data}
We denote vectors by lowercase letters, matrices by uppercase letters, and higher-order \textit{tensors} by calligraphic uppercase letters.  The \textit{order} of a tensor is the number of modes it admits. For example, $\T \in \R^{n_1 \times \cdots \times n_N}$ is an $Nth$-order tensor because it has N modes. A \textit{fiber} is the higher-order analogue of a matrix row or column in tensors. We show different ways to slice a third-order tensor in Figure~\ref{fig:tensor-norm}. \textit{The p-mode  matrix product} of a tensor $\T \in\R^{n_1 \times \cdots \times n_N}$ with a matrix $U \in \R^{m \times n_p}$ is denoted by $\T_{\times p}U $and is of size $n_1 \times \cdots n_{p-1} \times m\times n_{p+1}\times \cdots  \times n_N$. Element-wise it calculates: $(\T_{\times p}U)_{i_1\cdots i_{p-1} j i_{p+1} \cdots i_N} = \sum_{i_p = 1}^{n_p} \T_{i_1\cdots i_N}U_{ji_p}$.

\textbf{Higher-order count sketch(HCS)}
\label{sec:HCS-def}
Given a vector $u \in \R^{ d}$, random hash functions $h_{k}$:$[n_k]$ $\to$ $[m_k]$, $k \in [l]$, random sign functions $s_{k}$:$[n_k]$ $\to \{\pm 1\}$, $k\in [l]$, and $d = \prod_{k=1}^l n_k$, we propose HCS as:
\begin{equation}
\HCS(u)_{t_1, \cdots, t_l}  := \sum_{\mathclap{h_{1}(i_1)= t_1, \dotsc, h_{l}(i_l)= t_l}}s_{1}(i_1)\cdots s_{l}(i_l)u_{j}
\end{equation}
where $j = \sum_{k=2}^l{i_k\prod_{p=1}^{k-1}n_p }+i_1$. This is the index mapping between the vector with its reshaping result---a $lth$-order tensor with dimensions $n_k$ on each mode, for $ k \in [l]$.

Using tensor operations, we can denote HCS as:
\begin{equation}
\label{eqn:compact-ms}
\HCS(u) = (\mathcal{S} \circ reshape(u))_{\times 1}H_1 \dotsc_{\times l} H_l
\end{equation}
Here, $\mathcal{S} = s_1 \otimes \cdots \otimes s_l \in \R^{n_1 \times \cdots \times n_l}$, $H_k \in \R^{n_k \times m_k}$, $H_k(a,b) = 1$, if $h_k(a)=b$, otherwise $H_i(a,b) = 0$, for $\forall a \in [n_k], b \in [m_k], k \in [l]$. 
The $reshape(u)$ can be done in-place. We assume $u$ is a vectorized layout of a $lth$-order tensor. 

To recover the original tensor, we have 
\begin{equation}
\label{eqn:compact-HCS-vec-de}
\hat{u}_{j} = s_1(i_1)\cdots s_{l}(i_l)\HCS(u)_{h_1(i_1),\cdots,h_l(i_l)}
\end{equation}
Assume $\mathcal{T}_p$ is a $pth$-order tensor by fixing $l-p$ modes of a $lth$-order tensor $reshape(u)$ as shown in Figure~\ref{fig:tensor-norm}:
\begin{theorem}[\textbf{HCS recovery analysis}]
\label{thm:HCS-vec}
Given a vector $u \in \R^{d}$, assume $T_p$ is the maximum frobenium norm of all $\mathcal{T}_p$,
 Equation~\ref{eqn:compact-HCS-vec-de}
computes an unbiased estimator for $u_{j*}$ with variance bounded by:
\begin{equation}
\Var(\hat{u}_{j*}) = O(\sum_{p=1}^l\frac{T_p^2}{m^p})
\label{eqn:hcs-var}
\end{equation}

\end{theorem}

\textbf{Remarks}
Compared to CS, HCS requires less space for storing the hash functions.
Each mode only requires a $m_k \times n_k$ sized hash matrix with $n_k$ nonzero entries ($n_k = O(\sqrt[l]{d})$). Thus, HCS required $O(l\sqrt[l]{d})$ for hash memory while CS requires $O(d)$. If we choose $l = o(d)$, then $O(d) \gg O(l\sqrt[l]{d})$ and we save memory from using HCS.
 




The recovery variance calculates how likely input data collapse into the same place in the output. For example, in matrix case, to recover a specific data point, the variance depends on three parts: how likely the data from the same row but not same column, the data from the same column but not the same row and the data from different row and different column that get hashed into the same place as the specific data. The dominant term in Equation~\ref{eqn:hcs-var} will be ${\norm{u}_2^2}/{m^l}$ as long as all large entries are not clustered close to one another. Notice that CS has variance bounded by $\norm{u}_2^2/c$. We require $O(m^l) = O(c)$ to guarantee the same recovery, and that will lead to a total output memory $O(c)$ for HCS. 
In the worst case, when large magnitude data all locate in one fiber of $reshape(u)$, HCS has variance bounded by $\norm{u}_2^2/m$.  
 We require $O(m) = O(c)$ for the same recovery error. HCS output is of size $O(c^l)$ while CS output's size is $O(c)$.
 
We present a simple way to reshuffle the data in-place. Step1: Sort $u$ in  descending order.
Step2: Rearrange sorted array in designed space $n_1 \times \dotsc \times n_l$ such that it goes diagonally from top to bottom and then consecutive anti-diagonally from bottom to top.
Step3: Rearrange the data according to Step2 (column-wise fiber by fiber). 
We assume all data is well distributed in the rest analysis.

Another concern in HCS is how to choose the order of the reshaping tensor (parameter $l$). If the data values are fairly evenly distributed, we should select $l$ as large as possible (but sublinear in $d$). 

\section{Higher-order count sketch on tensors}

In the previous section, we discuss how to sketch a vector-valued data using HCS. In this section, we focus on tensor-valued data.
In order to use CS on higher-order tensors, we either view the tensor as a set of vectors and sketch along each fiber
 of the tensor or we flatten the tensor as a vector and apply CS on it.  Hence, CS
do not exploit tensors. Moreover, operations between tensors have to be performed on sketched vectors. This process is inefficient. But, with the help of HCS, we can compute  various operations such as  tensor products and tensor contractions by directly applying operations on the sketched components.

It is straightforward to apply HCS on tensors. Given a tensor $\T \in \R^{ n_1\times \cdots \times n_N}$, random hash functions $h_{k}$:$[n_k]$ $\to$ $[\msdim_k]$, $k \in [N]$, and random sign functions $s_{k}$:$[n_k]$ $\to \{\pm 1\}$, $k\in [N]$, HCS computes: $\HCS(\T) = (\mathcal{S} \circ \T)_{\times 1}H_1 \dotsc_{\times N} H_N$.
To recover the original tensor, we have: $\hat{\T} = \mathcal{S} \circ \HCS(\T)_{\times 1}H_1^{T}, \cdots_{\times N} H_N^{T}$.
$\mathcal{S}$ and $H_i$ are defined as same as in Section~\ref{sec:HCS-def}.

\subsection{Tensor product}
\textit{Tensor product} is known as outer product in vectors case.
It computes every bilinear composition from inputs. We denote the operation with $\otimes$. The tensor product result has dimension equal to the product of the dimensions of the inputs.
It has been used in a wide range of applications such as bilinear models~\citep{bilinear}.
\cite{cmm} shows that the count sketch of an outer product equals the convolution between the count sketch of each input vector as shown in Equation~\ref{eqn:cs-outerproduct}. Furthermore, the convolution in the time domain can be transferred to the element-wise product in the frequency domain. We extend the outer product between vectors to tensor product.

\begin{lemma}
\label{lemma:tensorproduct}
Given a $pth$-order tensor $\mathcal{A}$, a $qth$-order tensor $\mathcal{B}$, assume $p > q$: 
\begin{equation}
\begin{split}
\HCS(\mathcal{A} \otimes \mathcal{B}) &= \HCS(\mathcal{A}) * \HCS(\mathcal{B}) \\
&= IFFT(FFT(\HCS(\mathcal{A})) \circ FFT(\HCS(\mathcal{B})))  
\end{split}
\end{equation}
\end{lemma}
$FFT$ and $IFFT$ are p-dimensional Fourier transform and inverse Fourier transform if the input is a $pth$-order tensor. The proof is given in Appendix~\ref{app:krocnocker}.

We use the Kronecker product, which is a generalization of the outer product from vectors to matrices to compare tensor product approximation using HCS and CS. 

Assume inputs are $A,B \in \R^{n \times n}$: Given Lemma~\ref{lemma:tensorproduct}, this approximation requires $O(n^2)$ to complete $\HCS(A)$, $\HCS(B)$ and $O(m^2\log m)$ to complete 2D Fourier Transform if the HCS sketching parameter is $m$ for each mode.  It requires $O(m^2)$ memory for final representation and $O(n)$ for hashing parameters along each mode.

\textbf{Baseline CS operation}
We flatten $A$ and $B$ as vectors and apply CS on the vector outer product. 
The computation complexity is $O(n^2+c\log c)$ and the memory complexity is $O(c+ n^2)$. It needs $O(n^2)$ for hashing memory because we have $O(n^2)$ elements in the vectorized matrix. 

HCS requires approximately $O(n)$ times less memory comparing to CS for two $n \times n$ matrix Kronecker product. 
See Table~\ref{table:kron-summary} for detailed comparisons.

\begin{table}[ht]
\caption{Computation and memory analysis of various operation estimation (Results select sketch size to maintain the same recovery error for CS and HCS)}
\centering
\scalebox{0.95}{
\begin{tabular}{ |c c c|} 
 \hline
Operator & {Computation} & {Memory} \\
\hline
$\CS(A\otimes B)$  & $O(n^2+c\log c)$ &$O(c+ n^2)$\\ 
$\HCS(A\otimes B)$ &$ O(n^2+c\log c)$  &$O(c+n)$\\ 
\hline
$\CS(AB)$ &  $O(nr+cr\log c)$  & $O(c+ n+cr)$\\
$\HCS(AB)$  &$ O(nr+cr)$ &$O(c+n+\sqrt{c}r)$\\ 
 \hline
 $\CS(Tucker(\T))$ &$O(nr^3+cr^3\log c)$&$O(c+n+cr^3)$\\
  $\HCS(Tucker(\T))$&$O(nr+cr^3)$&$O(c+n+\sqrt[3]{c}r)$ \\
  \hline
\end{tabular}
}
\label{table:kron-summary}
\end{table}
\subsection{Tensor contraction}
\label{sec:tce}
Matrix product between $A \in \R^{m \times r}$ and $B \in \R^{r \times n}$ is defined as $C=AB= \sum_{i=1}^{r}A_{:i} \otimes B_{i:}$, $C \in \R^{m \times n} $.
\textit{Tensor contraction} (used more often as Einstein summation in physics community) can be seen as an extension of matrix product in higher-dimensions. It is frequently used in massive network computing.
We define a general tensor contraction between $\mathcal{A}\in \R^{a_1 \times \cdots \times a_p} $ and $\mathcal{B}\in \R^{b_1 \times \cdots \times b_q}$ as
\begin{equation}
    \mathcal{C}_{L}= \mathcal{A}_{P}\mathcal{B}_{Q}=   \mathcal{A}_{{M}{R}}\mathcal{B}_{{R}{N}} = \sum_{{R}}\mathcal{A}_{:{R}} \otimes \mathcal{B}_{{R}:}
    \label{eqn:tensorcontraction}
\end{equation}
 where ${P},{Q},{L}$ are ordered sequences of indices such that $P=\{a_1 \times \cdots \times a_p\}$, $Q=\{b_1 \times \cdots \times b_q\}$, $L = ({P} \cup {Q}) \backslash ({P} \cap {Q})$, $R = P \cap Q $, $M = {P} \backslash ({P} \cap {Q})$, ${N} = {Q} \backslash ({P} \cap {Q})$. The indices in ${R}$ are called
contracted indices. The indices in ${L}$ are called free indices. 

\begin{lemma}
\label{lemma:tensorcontraction}
Given tensors $\mathcal{A} \in \R^{{P}}$, $\mathcal{B}\in \R^{{Q}}$, contraction indices ${L}$, if hash matrices $H_i= I$, $ \forall i \in {L}$:
\begin{equation}
\HCS(\mathcal{A}_{P}\mathcal{B}_{Q}) = \HCS(\mathcal{A})\HCS(\mathcal{B})
\end{equation}
\end{lemma}
 We require the hash matrices for the contraction modes be identity matrices. In other words, we are not compressing along the modes that are being multiplied. The proof is in Appendix~\ref{app:ab}.

\textbf{Baseline CS operation} To apply CS on tensor contraction, we have to apply CS to each addition term in Equation~\ref{eqn:tensorcontraction}.
Take matrix product as an example:
\begin{equation}
    \CS(AB) = \CS(\sum_{i=1}^{r}A_{:i} \otimes B_{i:}) = \sum_{i=1}^{r} \CS(A_{:i} \otimes B_{i:}) = \sum_{i=1}^{r} \CS(A_{:i}) * \CS(B_{i:}) 
\end{equation}
We summarize computation and memory requirements for matrix product in Table~\ref{table:kron-summary}.

\subsection{Tucker-form tensor}

\textit{Tensor decomposition} is an extension of 
matrix decomposition to higher-orders. The Tucker decomposition~\citep{Tucker:1966:Psy} is analogous to principal component analysis. It decomposes a tensor as a core tensor contracted with a matrix along each mode. For instance, a third-order tensor $\T \in \R^{n_1 \times n_2 \times n_3}$ has the Tucker decomposition: $\T = \G_{\times 1}U_{\times 2} V_{\times 3} W$, where $\G \in \R^{r_1 \times r_2 \times r_3}$, $U \in \R^{n_1 \times r_1}$, $V \in \R^{n_2 \times r_2}$, $W \in \R^{n_3 \times r_3}$.
CANDECOMP/PARAFAC(CP)~\citep{Harshman:1970:UCLA} is a special case of a Tucker-form tensor, where the core tensor is a sparse tensor that only has non-zero values on the superdiagnoal. It can be represented as a sum of rank-1 tensors: $\T  = \sum_{i = 1}^r  \mathcal{G}_{iii}  U_i \otimes V_i \otimes W_i$.
Tensor decomposition has been applied in many field such as data mining~\citep{kolda-datamining} and latent variable models~\citep{animajmlr}.

\begin{lemma}
\label{lemma:tucker}
Given a Tucker tensor $\T = \G_{\times 1}U_{\times 2} V_{\times 3} W \in \R^{n \times n \times n}$, where 
$\mathcal{G }\in \R^{r \times r \times r}$:
The higher-order CS of a Tucker-form tensor can be accomplished by performing HCS on each factor:
\begin{equation}
\label{eqa:ms}
\HCS(\mathcal{T}) = \mathcal{G}_{\times 1} \HCS(U)_{\times 2} \HCS(V)_{\times 3} \HCS(W)
\end{equation}
\end{lemma}

\textbf{ Baseline CS operation} To apply CS to a Tucker-form tensor, we rewrite the decomposition as a sum of rank-1 tensors: 
\begin{align}
\label{eqn:cts-tucker}
\CS(\mathcal{T}) &= \sum_{a = 1}^r  \sum_{b = 1}^r  \sum_{c = 1}^r \mathcal{G}_{abc} \CS(U_a \otimes V_b \otimes W_c ) 
\end{align} 
where $U_a, V_b,W_c$ are $a^{th},b^{th},c^{th}$ column of $U,V,W$ respectively.

We show computation and memory analysis in Table~\ref{table:kron-summary}. In addition, a CP-form tensor can be sketched in the same way as we described above when using HCS. For using CS: instead of summing over all $\mathcal{G}$ values, we sum over only $r$ number of $\mathcal{G}$ values. The analysis can also be easily  extended to higher-order tensors.

We summarize the general tensor product and tensor contraction estimation process in Table~\ref{table:general-summary}.

\begin{table}[H]
\captionof{table}{General tensor operation estimation (Assume ${A}$ is a set of indices  with length $p$, ${B}$ is a set of indices with length $q$, each index value $O(n)$, assume the size of ${R}$ is $l$ with each index value $O(r)$, $g = \max(p,q)$) }
\centering
\begin{tabular}{ |c c  c |} 
 \hline
\multicolumn{3}{|c|}{\textbf{Tensor Product: $\mathcal{A} \in \R^{{A}}$, $\mathcal{B} \in \R^{{B}}$}}\\
 \hline
Operator & Computation & Memory \\
\hline
$\CS(\mathcal{A}\otimes \mathcal{B}) = \CS(vec(\mathcal{A})\otimes vec(\mathcal{B}) )$&$O(n^g+c\log c)$&$O(c+n^g)$\\
\hline
$\HCS(\mathcal{A}\otimes \mathcal{B}) = \HCS(\mathcal{A}) * \HCS(\mathcal{B})$&$O(n^g+c\log c)$&$O(c+gn)$\\
\hline
\multicolumn{3}{|c|}{\textbf{Tensor Contraction:} $\mathcal{A} \in \R^{{A}}$, $\mathcal{B} \in \R^{{B}}$ with contraction indices ${R}$} \\
\hline
Operator & Computation & Memory \\
\hline
$\CS(\mathcal{A} \mathcal{B}) = \sum_{{R}}\CS(A_{:{R}}\otimes B_{{R}:})$&$O(r^ln^g+cr^l\log c)$&$O(c+cr^l+n^g)$\\
\hline
$\HCS(\mathcal{A} \mathcal{B}) = \HCS(\mathcal{A}) \HCS(\mathcal{B})$&$O(r^ln^g+cr^l)$&$O(c+c^{\frac{g}{p+q}}r^l+gn)$\\
\hline
\end{tabular}
\label{table:general-summary}
\end{table}

\section{Experiments}


The goals of this section are: evaluate HCS for data compression; demonstrate the advantages of HCS in  various tensor operation estimations, compared to CS; present potential application of HCS in deep learning tasks. All synthetic data experiments are run on a MAC with Intel Core i5 processor. Section~\ref{sec:trl-exp} is run on a NVIDIA DGX station with Tesla V100 Tensor Core GPU.

\subsection{HCS for unevenly-distributed data}
In Section~\ref{sec:HCS-def}, we point out that unevenly-distributed data value may affect the performance of  HCS. We generate a matrix $A \in \R^{50 \times 50}$, where every entry is sampled from a uniform distribution between $-1$ and $1$, except the elements in the second column, which are filled with value $100$. We compress this data using HCS and CS. The compression ratio is calculated as $2500/m^2$ where $m$ is the sketching dimension along each two mode for HCS. CS sketching dimension $c = m^2$. We rearrange the data so that values in the matrix are evenly distributed, and we run the HCS and CS again. We compare the relative error($\frac{\norm{\hat{A}-A}_F}{\norm{A}_F}$) in Figure~\ref{fig:even-exp}. HCS performs poorly before rearranging the data. But after the rearrangment, HCS performs as good as CS, which corresponds to our analysis.

\begin{figure}[ht] 
  \begin{subfigure}[b]{0.33\linewidth}
    \centering
    \includegraphics[height = 3cm]{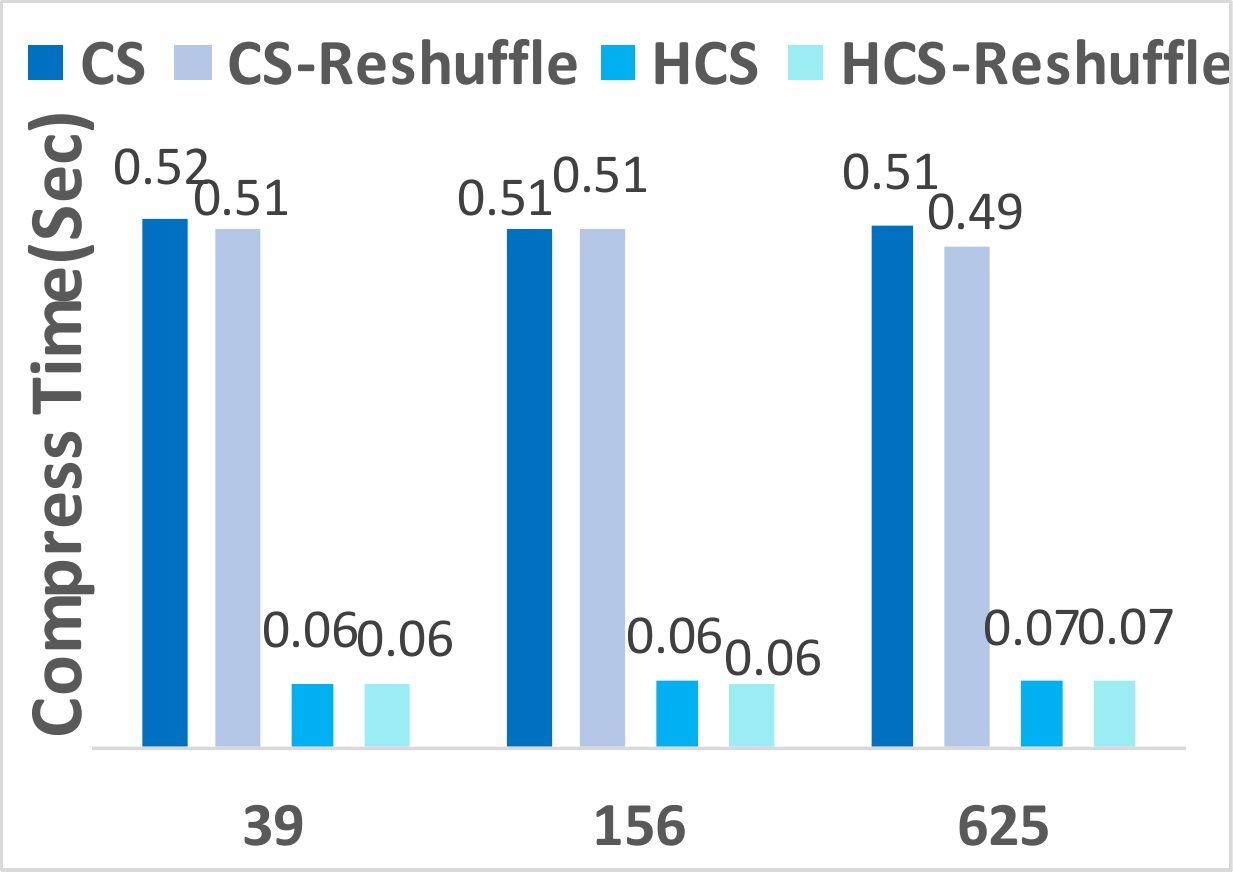}
\label{fig:even-time}
  \end{subfigure}
  \hfill
  \begin{subfigure}[b]{0.33\linewidth}
    \centering
        \includegraphics[height = 3cm]{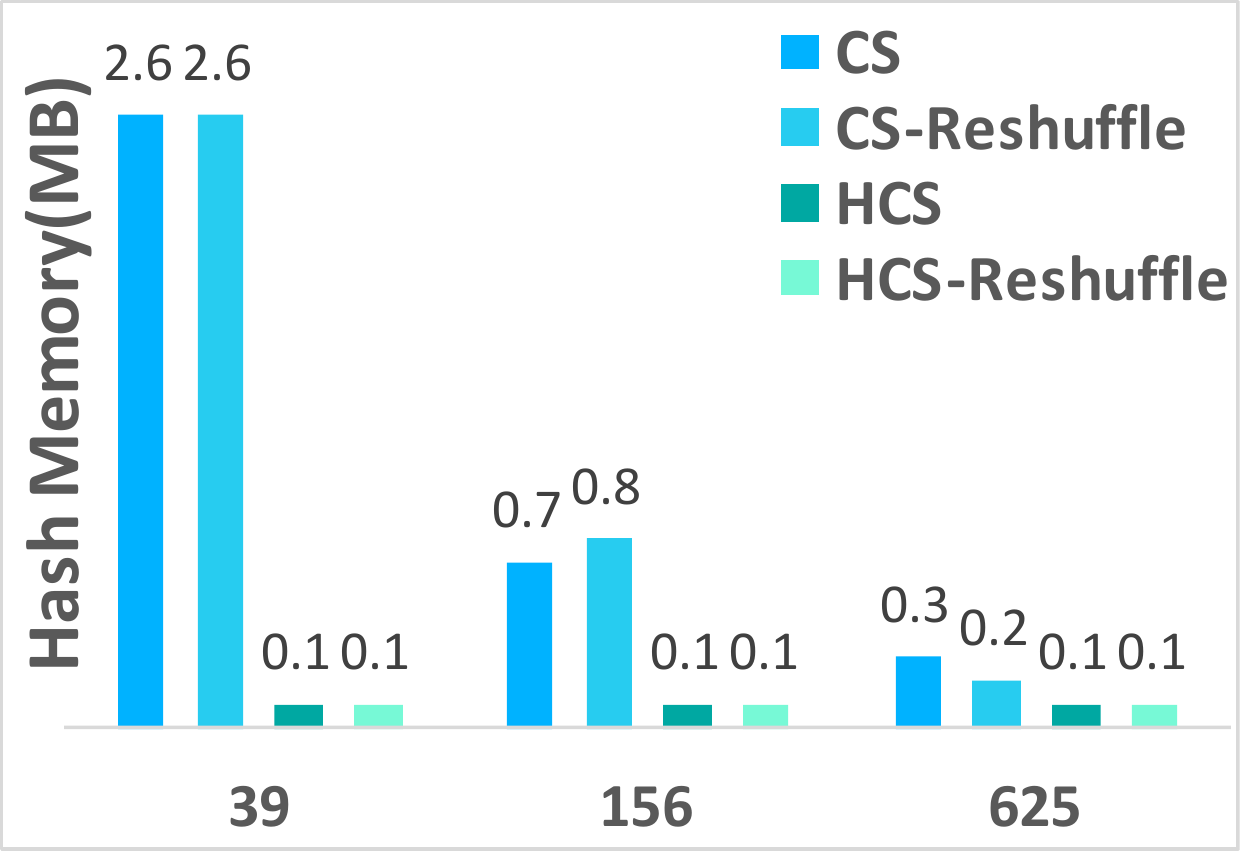}
\label{fig:even-memory}
  \end{subfigure} 
  \hfill
  \begin{subfigure}[b]{0.33\linewidth}
    \centering
\includegraphics[height = 3cm]{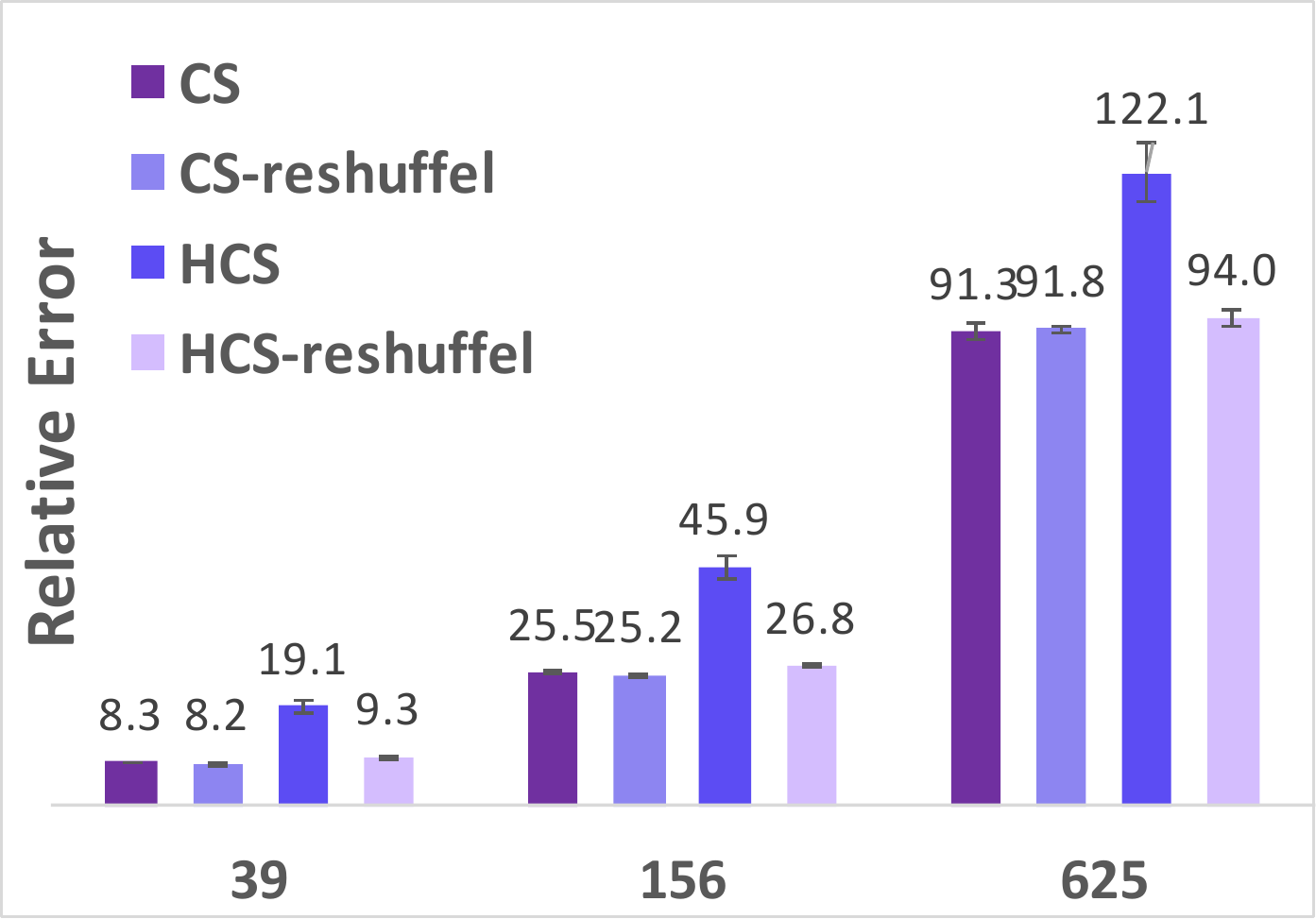}
\label{fig:even-error}
  \end{subfigure}
  \caption{Running time, memory and error comparison for unevenly-distributed data (x-axis shows the compression ratio).}
  \label{fig:even-exp}
\end{figure}

\subsection{Tensor operations}
\textbf{Kronecker product:}
We compress Kronecker products using HCS and CS.
We compute $A\otimes B$, where $A,B \in \R^{30 \times 30}$. All inputs are randomly generated from the uniform distribution[-5,5].
The result is obtained by independently running the sketching $20$ times and choosing the median. 
Keeping the same compression ratio, HCS has slightly higher recovery error than CS.
But HCS is systematically better in computation speed and memory usage compared to CS in Figure~\ref{fig:kron-exp}. 

\begin{figure}[ht] 
  \begin{subfigure}[b]{0.33\linewidth}
    \centering
    \includegraphics[width=4.5cm]{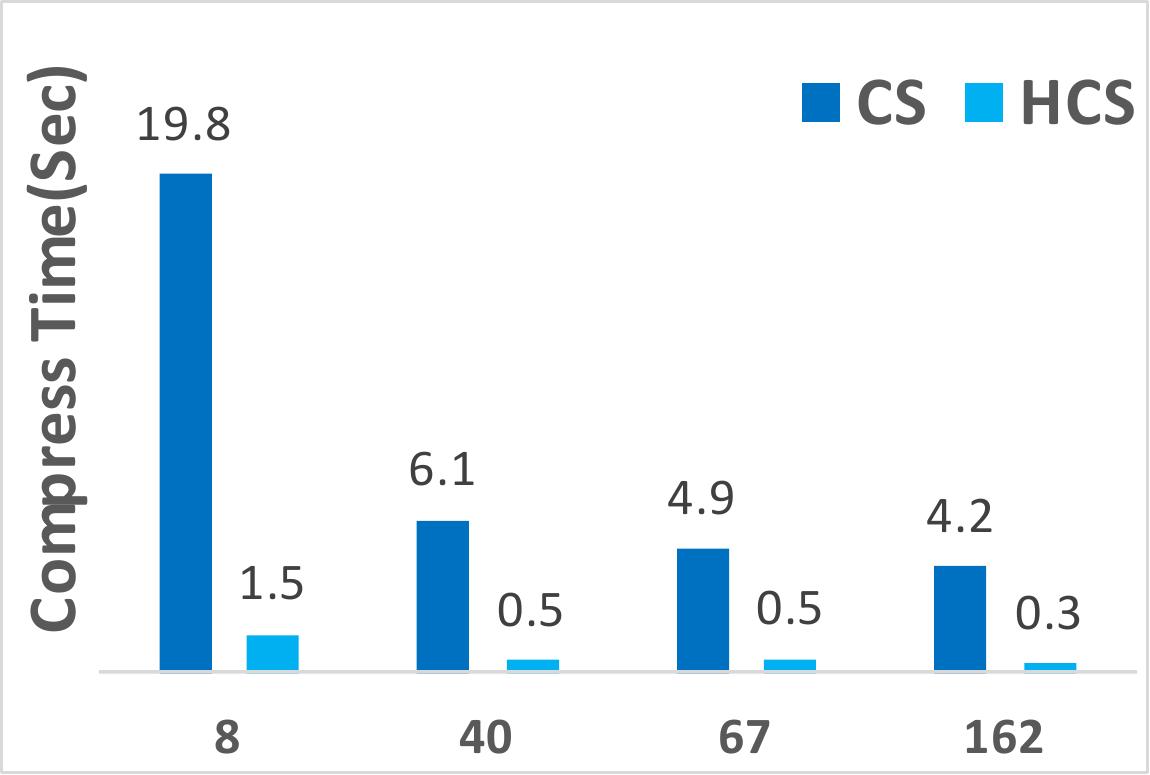}
\label{fig:kron-time}
  \end{subfigure}
  \hfill
  \begin{subfigure}[b]{0.33\linewidth}
    \centering
        \includegraphics[width=4.5cm]{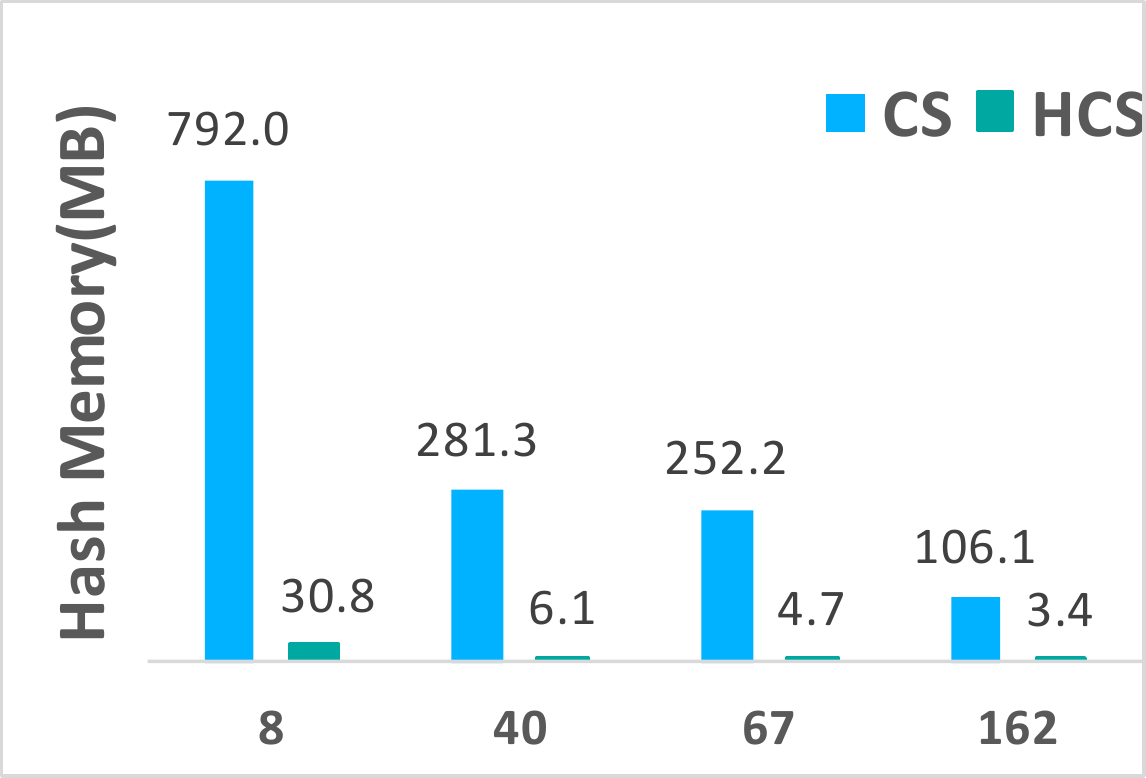}
\label{fig:kron-memory}
  \end{subfigure} 
  \hfill
  \begin{subfigure}[b]{0.33\linewidth}
    \centering
\includegraphics[width=4.5cm]{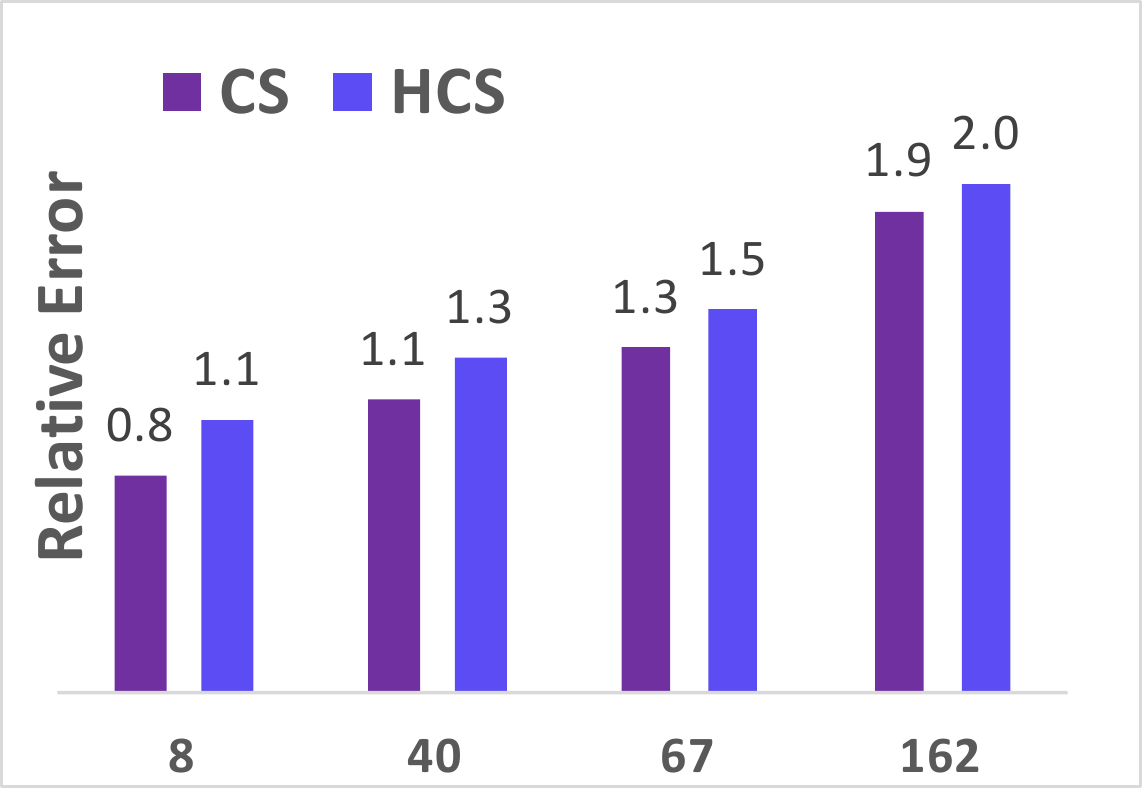}
\label{fig:kron-error}
  \end{subfigure}
  \caption{Running time, memory and error comparison for Kronecker product.}
  \label{fig:kron-exp}
\end{figure}

\textbf{Tensor contraction:}
\setlength{\columnsep}{0pt}%
Given $\mathcal{A}\in \R^{30 \times 30 \times 40}$, $\mathcal{B} \in \R^{40 \times 30 \times 30}$,
we compute $\mathcal{A}\mathcal{B}\in \R^{30 \times 30 \times 30 \times 30}$: the third mode of $\mathcal{A}$ contract with the first mode of $\mathcal{B}$. 
We compress and decompress the contraction as demonstrated in Section~\ref{sec:tce}.
All entries are sampled independently and uniformly from [0,10]. 
We repeat the sketching 20 times and use the median as the final estimation. Overall, HCS outperforms CS in time, memory and recovery error aspects as shown in Figure~\ref{fig:contraction-exp}. When the compression ratio is $8$, HCS is $200$x faster than CS and uses $40$x less memory while keeping almost the same recovery error. HCS is more efficient in real computation because it  performs compact contraction in matrix/tensor format, while CS requires computation on each slide of the input tensor.

\begin{figure}[ht] 
  \begin{subfigure}[b]{0.33\linewidth}
    \centering
    \includegraphics[width=4.5cm]{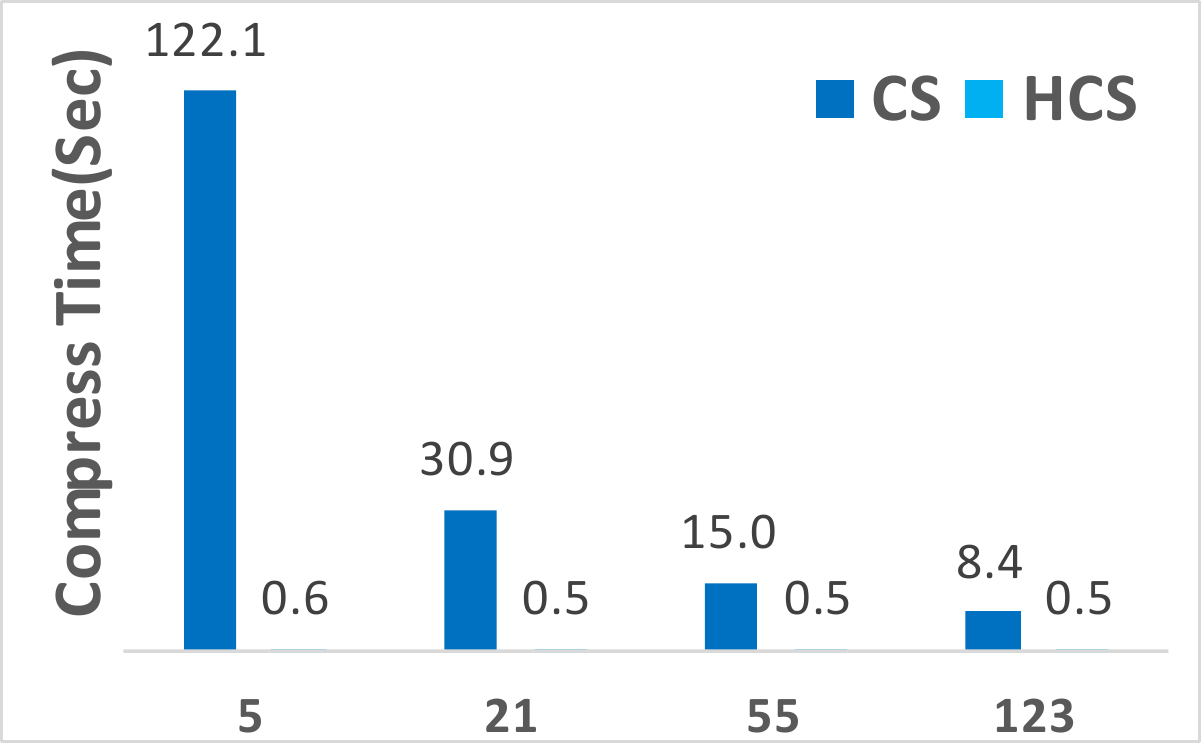}
\label{fig:con-time}
  \end{subfigure}
  \hfill
  \begin{subfigure}[b]{0.33\linewidth}
    \centering
        \includegraphics[width=4.5cm]{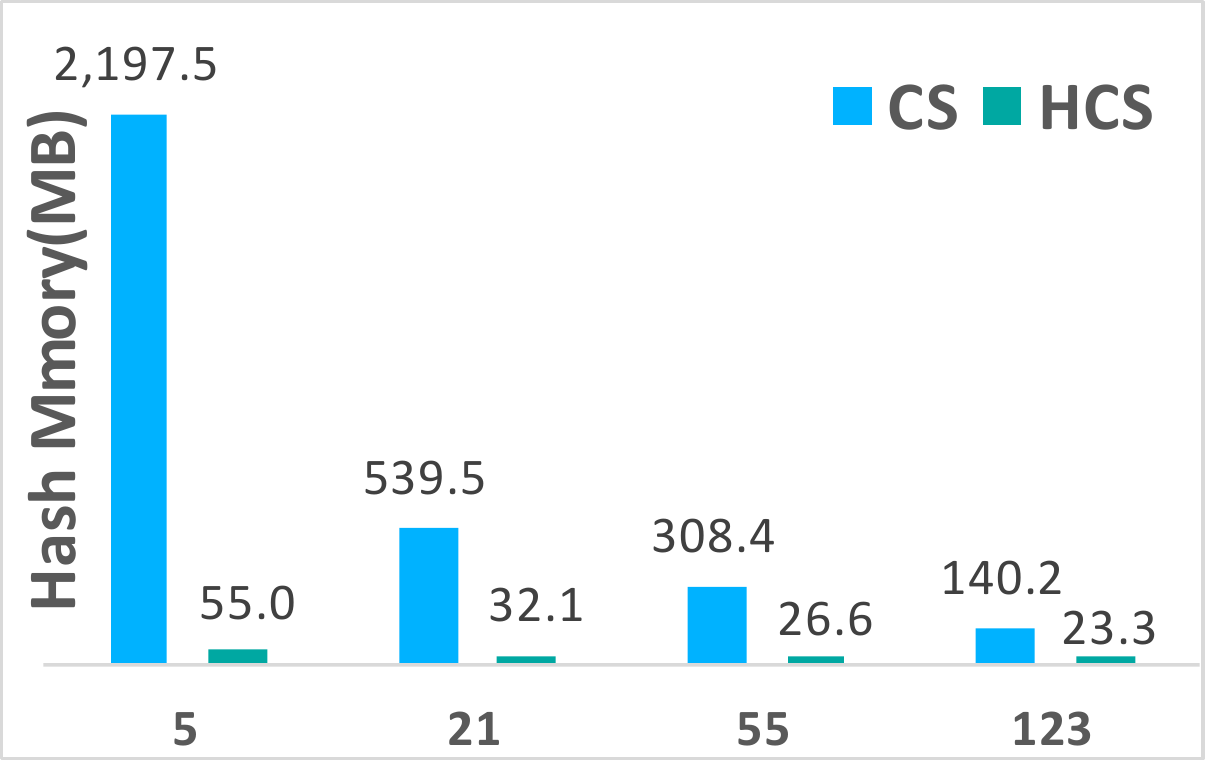}
\label{fig:con-memory}
  \end{subfigure} 
  \hfill
  \begin{subfigure}[b]{0.33\linewidth}
    \centering
\includegraphics[width=4.5cm]{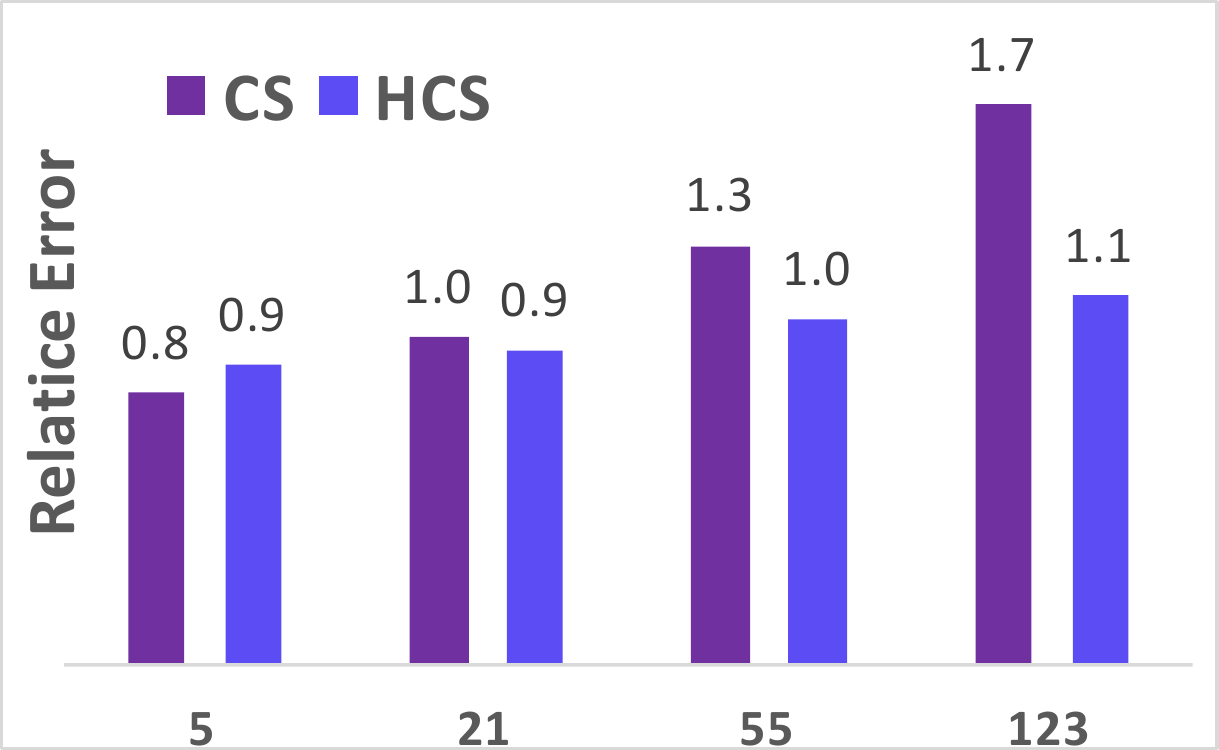}
\label{fig:con-error}
  \end{subfigure}
  \caption{Running time, memory and error comparison for tensor contraction.}
  \label{fig:contraction-exp}
\end{figure}

\begin{figure}[ht]
\begin{minipage}{0.55\textwidth}
\centering
    \includegraphics[height = 2.5cm]{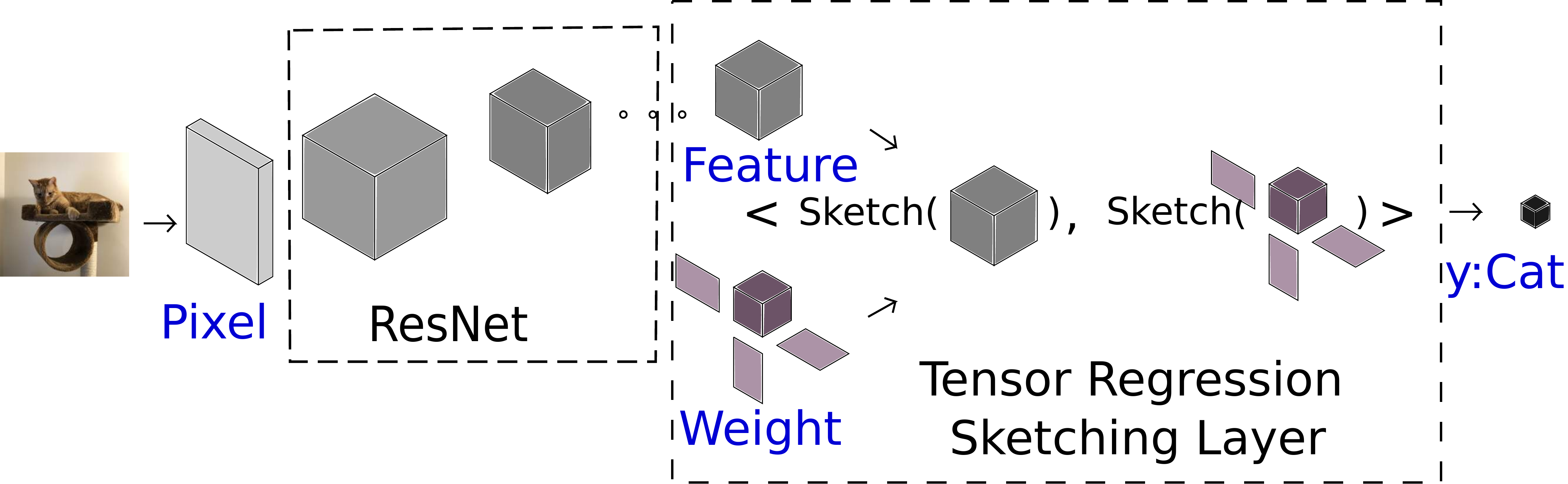}
    \caption{Tensor regression layer with sketching.}
    \label{fig:trl}
\end{minipage}
\hfill
\begin{minipage}{0.45\textwidth}
  \centering
    \includegraphics[height = 3.5cm]{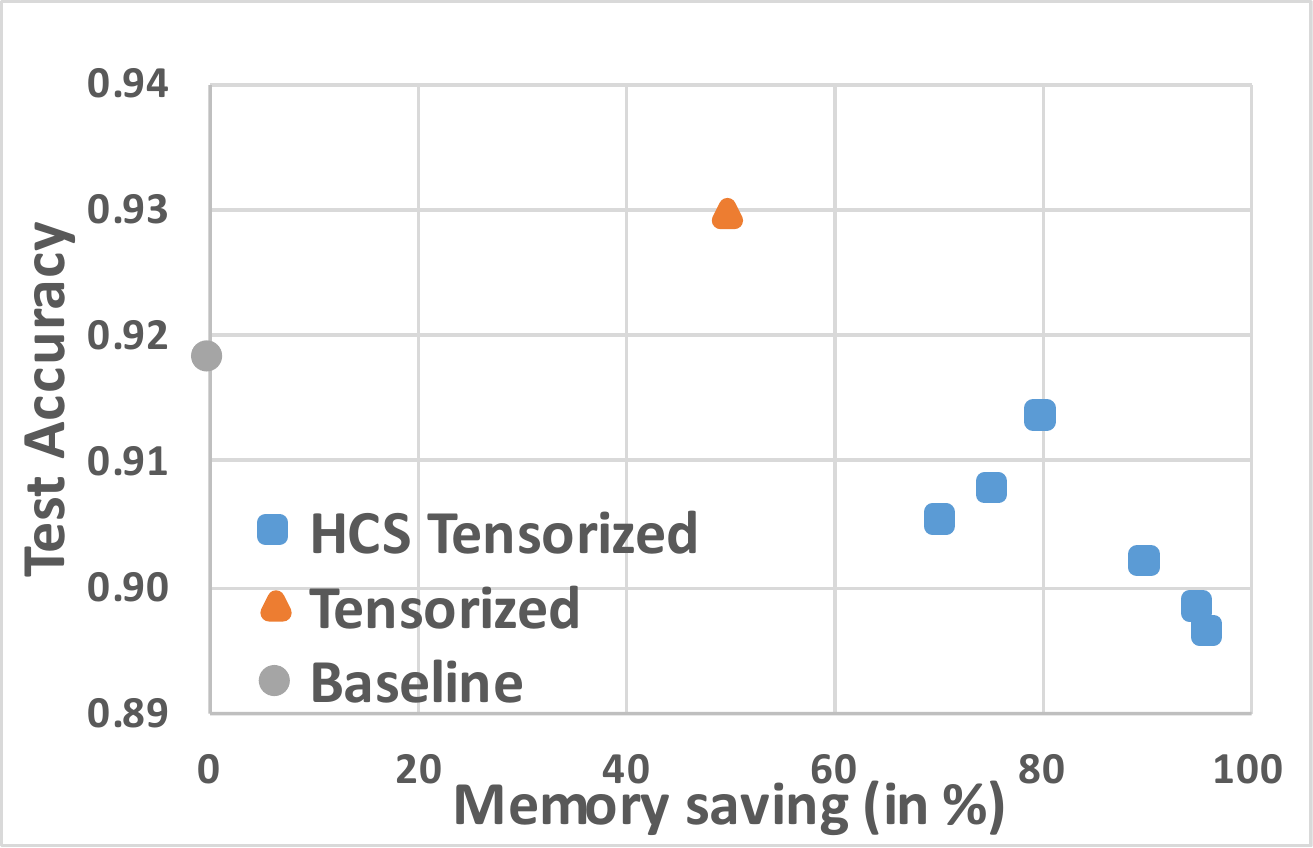}
    \captionof{figure}{Test accuracy on CIFAR 10.}
    \label{fig:cifar_compression}
\end{minipage}
\end{figure}

\subsection{Tensor regression network}
\label{sec:trl-exp}
To demonstrate the versatility of our method, we combine it by integrating it into a tensor regression network for object classification. 
Tensor regression layer (TRL)~\citep{jean} is proposed to 
learn a Tucker-form tensor weight for the high-order activation tensor. 
We sketch the Tucker tensor weight using Equation~\ref{eqa:ms}.
We use a ResNet18, from which the fully-connected layers are removed, and are replaced by our proposed sketched tensor regression layer. The network structure is illustrated in Figure~\ref{fig:trl}. The space saving is calculated as $1-\frac{P_T}{P_B}$ where $P_T$ and $P_B$ are the number of parameters in the last layer in the tensorized network and the baseline.
In Figure~\ref{fig:cifar_compression},  tensorized network outperforms the baseline(original Resnet18) while using $50\%$ less memory in the last layer. With HCS, we can further reduce $30\%$ more memory requirement  while keeping the prediction as good as the baseline.

 \section{Conclusion and future work}
In this paper, we extend count sketch to a new sketching technique, called higer-order count sketching (HCS). 
HCS gains an exponential saving (with respect to the order of the tensor) in the memory requirements of the hash functions and allows efficient approximation of various tensor operations such as tensor products and tensor contractions. 
Some interesting future works are how to choose the optimal tensor order for input (vector) data when we have limited information about the data and how to extend other hash algorithms such as simhash~\citep{Charikar:2002:SET:509907.509965}, minhash~\citep{one-bit} and cuckoo hashing~\citep{cuckoo} to tensors. We are also interested in analyzing the performance differences on sparse and dense tensors using various sketching techniques. Providing HCS  implementations within computation platforms such as MKL and CUDA is also part of the future work.
\section*{Acknowledgement}
This paper is supported by AFOSR Grant FA9550-15-1-0221.

\newpage
\bibliography{main}
\bibliographystyle{apalike}
\newpage

\begin{appendices}
\section{List of some algorithms mentioned in the main paper}
\label{app:alg}
\subsection{Count sketch}
\begin{algorithm}[h]
\caption{Count Sketch~\cite{cs}}
\label{alg:cs-single}
\begin{algorithmic}[1]
\Procedure{CS}{$x,c$}
\Comment{$x\in \R^{n}$}
\State \texttt{$s \in Maps([n]\Rightarrow \{-1,+1\})$}
\State\texttt{$h \in Maps([n]\Rightarrow \{0,\cdots c\})$}
\For{\texttt{i:1 to n} }
\State
\texttt{$y[h[i]] += s[i]x[i]$}
\EndFor
\State \textbf{return} $y$
\EndProcedure
\Procedure{CS-Decompress}{$y$}
\For{\texttt{i:1 to n} }
\State
\texttt{$\hat{x}[i] = s[i]y[h[i]]$}
\EndFor
\State \textbf{return} $\hat{x}$
\EndProcedure
\end{algorithmic}
\end{algorithm}
\subsection{Higher-order count sketch}
\begin{algorithm}[h]
\caption{Higher-order Count Sketch}
\label{alg:matrix-ms}
\begin{algorithmic}[1]
\Procedure{HCS}{$\T,M_{list}$}
\Comment{$\T\in \R^{n_1\times \cdots \times  n_N}$}
\\
\Comment{$M_{list} $ contains sketching parameters: $m_1 \dots m_N$}
\State
Generate hash functions $s_1, \cdots s_N$, $h_1, \cdots h_N$ given $M_{list}$
\State
Compute hash matrices $\mathcal{S}$, $H_1, \cdots H_N$


\State \textbf{return} $ (\mathcal{S} \circ \T)(H_1, \cdots, H_N)$
\EndProcedure
\Procedure{HCS-decompress}{$\HCS(\T)$}
\State \textbf{return} $\mathcal{S} \circ \HCS(\T)(H_1^{T}, \cdots, H_N^{T})$
\EndProcedure
\end{algorithmic}
\end{algorithm}

\subsection{Approximate Kronecker product}
\label{alg:kron}
\scalebox{0.90}{
\begin{minipage}{\linewidth}
\begin{algorithm}[H]
\caption{Compress/Decompress Kronecker Product}\label{alg:krocnecker}
\begin{algorithmic}[1]
\Procedure{Compress-KP}{$A,B,m_1,m_2$}\Comment{$A\in \R^{n_1\times  n_2}, B \in \R^{n_3 \times n_4}$}
\For{X in [A,B]}
\State
\texttt{$X^{HCS}$ = $\HCS$($X,[m_1,m_2]$)}
\EndFor

\State
\texttt{$\FFT2$($A^{HCS}$),$\FFT2$($B^{HCS}$)}
\State
\texttt{$P$=$\IFFT2$($A^{HCS}$ $\circ$ $B^{HCS}$)}
\State
\textbf{return} ($P$)
\EndProcedure
\\
\Procedure{Decompress-KP}{$P$}
\State
\texttt{$C$ = zeros($n_1n_3,n_2n_4$)}
\For{\texttt{w,q,o,g:=1 to $n_1,n_2,n_3,n_4$ }}
\State
\texttt{k = $(h_{A1}[w]+h_{B1}[o])\text{ mod } m_1$}
\State
\texttt{l = $(h_{A2}[q]+h_{B2}[g])\text{ mod } m_2$}
\State
\texttt{tmp = $s_{A1}[w]s_{A2}[q]s_{B1}[o]s_{B2}[g]P[k,l]$}
\State
\texttt{i = $n_3(w-1)+o$}
\State
\texttt{j = $n_4(q-1)+g$}
\State
\texttt{$C_{ij}$ = tmp}
\EndFor
\State \textbf{return} ($C$)
\EndProcedure
\end{algorithmic}
\end{algorithm}
\end{minipage}
}

\subsection{Approximate Matrix product}
\label{alg:mm}
\scalebox{0.90}{
\begin{minipage}{\linewidth}
\begin{algorithm}[H]
\caption{Compress/Decompress Matrix Product}\label{alg:matrixcontraction}
\begin{algorithmic}[1]
\Procedure{Compress-MP}{$A,B,m_1,m_2,m_3$}\Comment{$A\in \R^{n_1\times  k}, B \in \R^{k \times n_2}$}
\State
\texttt{$A^{HCS}$ = $\HCS$($A,[m_1,m_2]$)}\Comment{Choose hash matrix along k mode be identity matrix }
\State
\texttt{$B^{HCS}$ = $\HCS$($B,[m_2,m_3]$)}
\State
\texttt{$P$=$A^{HCS}$ $B^{HCS}$}
\State
\textbf{return} ($P$)
\EndProcedure
\\
\Procedure{Decompress-MP}{$P$}
\State
\texttt{$C$ = zeros($n_1,n_2$)}
\For{\texttt{i,j:=1 to $n_1,n_2$ }}
\State
\texttt{k = $h_{A1}[i]$}
\State
\texttt{l = $h_{B2}[j]$}
\State
\texttt{$C_{ij}$ = $s_{A1}[i]s_{B2}[j]P[k,l]$}
\EndFor
\State \textbf{return} ($C$)
\EndProcedure
\end{algorithmic}
\end{algorithm}
\end{minipage}
}

\section{Proofs of some technical theorems/lemmas}
\label{app:1}

\subsection{Analysis of CS and HCS approximation error}
\label{app:csms}
\begin{theorem}[\cite{cs}]
\label{thm:1}
Given a vector $u \in \R^{n}$, \textbf{CS} hashing functions s and h with sketching dimension \csdim, for any $i^*$, the recovery function $\hat{u}_{i^*} = s(i^*)CS(u)(h(i^*))$ computes an unbiased estimator for $u_{i^*}$ with variance bounded by $||u||^2_2/\csdim$.
\end{theorem}
\begin{proof}[\textbf{Proof of Theorem~\ref{thm:1}}]
For $i \in \{ 1,2,\cdots n \},$ let $K_i$ be the indicator variable for the event $h(i) = h(i^*)$. We can write $\hat{u}_{i^*}$ as
\begin{align}
\hat{u}_{i^*} = s(i^*)\sum_i K_i s(i) u_i
\end{align}
Observe that $K_i = 1$, if $i = i^*$, $\E(s(i^*)s(i)) = 0$, for all $i \neq i^*$, and $\E(s(i^*)^2) = 1$, we have
\begin{equation}
\begin{split}
\E(\hat{u}_{i^*}) &= \E(s(i^*)K_{i^*} s(i^*) u_{i^*}) + \E(s(i^*)\sum_{i \neq i^*} K_i s(i) u_{i}\\
&=u_i
\end{split}
\end{equation}
To bound the variance, we rewrite the recovery function as
\begin{equation}
\hat{u}_{i^*} = s(i^*)K_{i^*} s(i^*) u_{i^*} + s(i^*)\sum_{i \neq i^*} K_i s(i) u_i
\end{equation}
To simplify notation, we assign X as the first term, Y as the second term.
$\Var (X) = 0$, and $COV(X,Y) = 0$ since $s(i)$ for $i \in \{ 1,2,\cdots n \}$ are 2-wise independent. Thus,
\begin{equation}
\Var (X+Y) = \sum_{i \neq i^*}\Var( K_i s(i^*)s(i) u_i)
\end{equation}
$\E(K_i s(i^*)s(i) u_i) = 0$ for $i \neq i^*$. Consequently,
\begin{equation}
\Var( K_i s(i^*)s(i) u_i) = \E((K_i s(i^*)s(i) u_i)^2) = \E(K_i^2)u_i^2 = u_i^2/\csdim
\end{equation}
The last equality uses that $\E(K_i^2) = \E(K_i) = 1/\csdim$, for all $i \neq i^*$.
Summing over all terms, we have $\Var(\hat{u}_{i^*}) \leq ||u||^2_2/\csdim$.
\end{proof}

\begin{proof}[\textbf{Proof of Theorem~\ref{thm:HCS-vec}}]
For simplicity, we assume $u\in \R^{d}$ is reshaped into a  second-order tensor $A \in \R^{n_1 \times n_2}$ in the following proof. But the analysis can be extended to reshaping $u$ into any order tensor. \\
For $i \in \{ 1,2,\cdots n_1 \}$, $j \in \{ 1,2,\cdots n_2 \}$, let $K_{ij}$ be the indicator variable for the event $h_1(i) = h_1(i^*)$ and $h_2(j) = h_2(j^*)$. We can write $\hat{A}_{i^*j^*}$ as
\begin{equation}
\hat{A}_{i^*j^*} = s_1(i^*)s_2(j^*)\sum_{ij} K_{ij} s_1(i) s_2(j)A_{ij}
\end{equation}
Notice that $A = reshape(u)$, we know the index mapping: $A_{i^*j^*} = u_{t^*}$, where $t^* = n_2i^{*}+j^*$.
Observe that $K_{ij} = 1$, if $i = i^*$, $j = j^*$. $\E(s_1(i^*)s_1(i)) = 0$, $\E(s_2(j^*)s_2(j)) = 0$, for all $i \neq i^*$, $j \neq j^*$, and $\E(s_1(i^*)^2) = 1$, $\E(s_2(j^*)^2) = 1$, we have
\begin{equation}
\begin{split}
\E(\hat{A}_{i^*j^*}) &= \E(s_1^2(i^*)s_2^2(j^*)K_{i^*j^*} A_{i^*j^*} + \E(s_1(i^*)s_2(j^*)\sum_{i \neq i^* or j \neq j^*} K_{ij} s_1(i)s_2(j) A_{ij}) \\
&=A_{i^*j^*}
\end{split}
\end{equation}
To bound the variance, we rewrite the recovery function as
\begin{equation}
\hat{A}_{i^*j^*} = s_1^2(i^*)s_2^2(j^*)K_{i^*j^*} A_{i^*j^*} + s_1(i^*)s_2(j^*)\sum_{i \neq i^* or j \neq j^*} K_{ij} s_1(i)s_2(j) A_{ij}
\end{equation}
To simplify notation, we assign X as the first term, Y as the second term.
$\Var (X) = 0$, and $COV(X,Y) = 0$ since $s_1(i)$ and $s_2(j)$ for $i \in \{ 1,2,\cdots n_1 \}$, $j \in \{ 1,2,\cdots n_2 \}$ are both 2-wise independent. Thus,
\begin{equation}
\Var (X+Y) = \Var(X)+ \Var(Y)-2\Cov(X,Y) = \sum_{i \neq i^*or j \neq j^*}\Var( K_{ij}s_1(i^*)s_2(j^*)  s_1(i)s_2(j) A_{ij})
\label{eqn:var}
\end{equation}
$\E(K_{ij} s_1(i^*)s_2(j^*) s_1(i)s_2(j) A_{ij}) = 0$ for $i \neq i^*$ or $j \neq j^*$. 
Therefore, Equation~\ref{eqn:var} becomes:
\begin{equation}
\begin{split}
\sum_{i \neq i^*or j \neq j^*} \E((K_{ij} s_1(i^*)s_2(j^*)s_1(i)s_2(j) A_{ij})^2)&= \sum_{i \neq i^*or j \neq j^*}\E(K_{ij}^2)A_{ij}^2 \\
&= \sum_{i \neq i^*, j \neq j^*}\frac{A_{ij}^2}{m_1 m_2}
+\sum_{i \neq i^*,j =j^*}\frac{A_{ij}^2}{m_1}+\sum_{i = i^*, j \neq j^*}\frac{A_{ij}^2}{m_2}
\end{split}
\label{eqn:var2}
\end{equation}
This is because $\E(K_{ij}^2) = \E(K_{ij}) = 1/(m_1 m_2)$, for all $i \neq i^*$, $j \neq j^*$.
$\E(K_{ij}^2) = \E(K_{ij}) = 1/(m_1)$, for all $i \neq i^*$, $j = j^*$.
$\E(K_{ij}^2) = \E(K_{ij}) = 1/( m_2)$, for all $i = i^*$, $j \neq j^*$.

If any fiber of $A$ has extreme large data value, or $\max(\norm{A_i}_2) \approx \norm{u}_2$, where $A_i$ is any row or column of $A$, we can omit the first term, $\Var(\hat{A}_{i^*j^*}) \leq \norm{u}^2_F/(min(\msdim_1, \msdim_2))$. Otherwise, if $\norm{u}_2 \gg \max(\norm{A_i}_2)$, we can omit the second and third terms and $\Var(\hat{A}_{i^*j^*}) = \Omega( \norm{u}^2_2/(m_1 m_2))$. 
\end{proof}
\subsection{HCS of the Kronecker product}
For simplicity, we show proof for Kronecker product here. But this can be extended to general tensor product.
\label{app:krocnocker}
\begin{lemma}
\label{lemma:aob}
Given two matrices $A \in \R^{n \times n}$, $B \in \R^{n \times n}$, 
\begin{equation}
\begin{split}
HCS(A \otimes B) &= HCS(A) * HCS(B) \\
&= IFFT2(FFT2(HCS(A)) \circ FFT2(HCS(B)))  
\end{split}
\end{equation}
\end{lemma}
\begin{proof}[\textbf{Proof of Lemma~\ref{lemma:aob}}]
The Kronecker product defines $(A \otimes B)_{n_3(p-1)+h\ n_4(q-1)+g} = A_{pq}B_{hg}$. Thus:
\begin{align}
& \sum_{pqhg}(A \otimes B)_{ab} 
s_{1}(p)s_{2}(q)s_3(h)s_4(g) w^{t_1h_a+t_2h_b} \nonumber \\
& = \sum_{pqhg}A_{pq}B_{hg}s_{1}(p)s_{2}(q)s_3(h)s_4(g) w^{t_1h_a+t_2h_b} \nonumber \\
& = \sum_{pq}A_{pq}s_1(p)s_2(q)w^{t_1 h_1(p)+t_2 h_2(q)} \sum_{hg}B_{hg}s_3(h)s_4(g)w^{t_1 h_3(h)+t_2 h_4(g)} \nonumber \\
& = FFT2(HCS(A)) \circ FFT2(HCS(B))
\end{align}
where $a = n_{3}(p-1)+h$, $b =n_4(q-1)+g$, $h_a = h_1(p)+h_3(h)$, $h_b = h_2(q)+h_4(g)$.\\
Assign $i = n_{3}(p-1)+h$, $j = n_4(q-1)+g $, $s_5(i) = s_1(p)s_3(h)$, $s_6(j) = s_1(q)s_3(g)$, $h_5(i) = h_1(p)+h_3(h)$ and $h_6(i) = h_2(q)+h_4(g)$, we have
\begin{align}
& \sum_{pqhg}(A \otimes B)_{ab} 
s_{1}(p)s_{2}(q)s_3(h)s_4(g)w^{t_1h_a+t_2h_b} \nonumber \\
& = \sum_{ij} (A \otimes B)_{ij} 
s_{5}(i)s_{6}(j)w^{t_1h_5(i)+t_2h_6(j)} \nonumber \\
& = FFT2(HCS(A \otimes B)) \nonumber \\
& = FFT2(HCS(A)) \circ FFT2(MS(B))
\end{align}
Consequently, we have $HCS(A \otimes B) =  IFFT2(FFT2(HCS(A)) \circ FFT2(HCS(B)))$. The recovery map is $\hat{A\otimes B}_{n_3(p-1)+h \ n_4(q-1)+g} = s_1(p)s_2(q)s_3(h)s_4(g)HCS(A\otimes B)_{(h_1(p)+h_3(h)) \textmd{mod} \ m_1 \ (h_2(q)+h_4(g)) \textmd{mod} \ m_2}$
for $p \in [n_1]$, $q \in [n_2]$, $h \in [n_3]$, $g \in [n_4]$.
\end{proof}

\subsection{HCS of the matrix product}
\label{app:ab}
 Higher-order tensor contraction can be seen as a matrix product by grouping all free indices and contraction indices separately. We show the proof for Lemma~\ref{lemma:tensorcontraction} in matrix case.
\begin{lemma}
\label{lemma:ab}
Given two matrices $A \in \R^{n \times n}$, $B \in \R^{n \times n}$, 
$\HCS(A) = H_1(s_1 \otimes s_2 \circ A)H_2^T$, $\HCS(B) = H_2(s_2 \otimes s_3 \circ B)H_3^T$, then
\begin{equation}
\HCS(AB) = \HCS(A)\HCS(B)
\end{equation}
if $H_2^TH_2 = I$.
\end{lemma}

\begin{proof}[\textbf{Proof of Lemma~\ref{lemma:ab}}]
The compact HCS representations for A and B are
$\HCS(A) = H_1(s_1 \otimes s_2 \circ A)H_2^T$, $\HCS(B) = H_2(s_2 \otimes s_3 \circ B)H_3^T$
as described in Section~\ref{sec:HCS-def}. Here $H_1 \in \R^{m_1 \times n_1}$, $H_2 \in \R^{m_2 \times r}$, $H_3 \in \R^{m_3 \times n_2}$, $s_1 \in \R^{n_1}$, $s_2 \in \R^{r}$ and $s_3 \in \R^{n_2}$. 
Assume $\HCS(AB) = H_4(s_4 \otimes s_5 \circ AB)H_5^T$.

If $H_2$ is orthogonal, or $H_2^TH_2 = I$,
\begin{equation}
\begin{split}
    \HCS(A)\HCS(B) & = H_1(s_1 \otimes s_2 \circ A)H_2^T H_2(s_2 \otimes s_3 \circ B)H_3^T  \\
    & = H_1(s_1 \otimes s_2 \circ A)(s_2 \otimes s_3 \circ B)H_3^T  \\
    & = H_1(s_1 \otimes s_3 \circ AB)H_3^T
    \end{split}
\end{equation}

By setting $H_4 = H_1$, $H_5 = H_3$, $s_4 = s_1$ and $s_5 = s_3$, we have $\HCS(AB) = \HCS(A)\HCS(B)$.
\end{proof}

\subsection{Analysis of Kronecker product approximation error}
\label{app:kronecker-error}
\begin{theorem}[\textbf{CS recovery analysis for Kronecker product}]
\label{thm:kronecker-ts}
Suppose $\hat{C}$ is the recovered tensor for $C = A \otimes B$ after applying CS on $A \otimes B$ with sketching dimension $c$.  We suppose the estimation takes $d$ independent sketches of $A \otimes B$ and then report the median of the $d$ estimates.
If $d = \Omega(\log(1/\delta))$,  $c = \Omega(\frac{\norm{C}_F^2}{\epsilon^2})$, then with probability $\geq 1- \delta$ there is $|\hat{C}_{ij}-C_{ij}| \leq \epsilon$. 
\end{theorem}
\begin{proof}
$\CS(C) = \CS(vec(A) \otimes vec(B))$. Given Theorem~\ref{thm:1}, we have $\E(\hat{C}) = C = vec(A) \otimes vec(B)$, $\Var(\hat{C}_{ij}) \leq \norm{vec(A)\otimes vec(B)}_2^2/c = \norm{C}_F^2/c$. 
From Chebychev’s inequality, if we run this sketch $d$ times, where $d = \Omega(\log(1/\delta))$, we can get the desired error bond with probability at least $1-\delta$.
\end{proof}

\begin{theorem}[\textbf{HCS recovery analysis for Kronecker product}]
\label{thm:kronecker-HCS}
Suppose $\hat{C}$ is the recovered tensor for $C = A \otimes B$ after applying HCS on $A \otimes B$ with sketching dimension $m$ along each mode.  We suppose the estimation takes $d$ independent sketches of $A \otimes B$ and then report the median of the $d$ estimates.
If $d = \Omega(\log(1/\delta))$,  $m^2 = \Omega(\frac{\norm{C}_F^2}{\epsilon^2})$, then with probability $\geq 1- \delta$ there is $|\hat{C}_{ij}-C_{ij}| \leq \epsilon$. 
\end{theorem}
\begin{proof}
 We have shown in Lemma~\ref{lemma:aob} that $\HCS(C) = \HCS(A) * \HCS(B)$. Given Theorem~\ref{thm:HCS-vec}, we have $\E(\hat{C}) = C = A \otimes B$, $\Var(\hat{C}_{ij}) \leq \norm{C}_F^2/m^2$ (We assume $C$ is well-distributed). 
From Chebychev’s inequality, if we run this sketch $d$ times, where $d = \Omega(\log(1/\delta))$, we can get the desired error bond with probability at least $1-\delta$.
\end{proof}

\subsection{Analysis of matrix product approximation error}
\label{app:mm-error}
\begin{theorem}[\textbf{CS recovery analysis for matrix product}]
\label{thm:mm-ts}
Suppose $\hat{C}$ is the recovered tensor for $C = AB$ after applying CS on $AB$ with sketching dimension $c$.  We suppose the estimation takes $d$ independent sketches of $AB$ and then report the median of the $d$ estimates.
If $d = \Omega(\log(1/\delta))$,  $c = \Omega(\frac{\norm{C}_F^2}{\epsilon^2})$, then with probability $\geq 1- \delta$ there is $|\hat{C}_{ij}-C_{ij}| \leq \epsilon$. 
\end{theorem}
\begin{proof}
$\CS(C) = \sum_{i=1}^{r} \CS(A_i \otimes B_i)$.
Thus, 
\vspace{-0.5em}
\begin{equation}
\E(\hat{C}) = \sum_{k=1}^{r} \E(CS(A_k \otimes B_k)) = \sum_{k=1}^{r} A_k \otimes B_k = C
\end{equation}
\vspace{-1em}
\begin{equation}
\begin{split}
\Var(\hat{C}_{ij}) &= \sum_{k=1}^{r} \Var((\hat{A}_{ik} \hat{B}_{kj})) \\
& = \sum_{k=1}^{r} \E^2(\hat{A}_{ik}) \Var(\hat{B}_{kj})+\E^2(\hat{B}_{kj}) \Var(\hat{A}_{ik})+\Var(\hat{A}_{ik})\Var(\hat{B}_{kj}) \\
& \leq \sum_{k=1}^{r} A_{ik}\norm{B_k}_2^2/c+B_{kj}\norm{A_k}_2^2/c+\norm{A_k}_2^2\norm{B_k}_2^2/c^2 \\
& \leq \sum_{k=1}^{r}\norm{A_k}_2^2\norm{B_k}_2^2(\frac{1}{c}+\frac{1}{c^2}) \\
& \leq 3\norm{AB}_F^2/c
\end{split}
\end{equation}
From Chebychev’s inequality, if we run this sketch $d$ times, where $d = \Omega(\log(1/\delta))$, we can get the desired error bond with probability at least $1-\delta$.
\end{proof}

\begin{theorem}[\textbf{HCS recovery analysis for matrix product}]
\label{thm:mm-HCS}
Suppose $\hat{C}$ is the recovered tensor for $C = AB$ after applying HCS on $AB$ with sketching dimension $m$ along each mode.  We suppose the estimation takes $d$ independent sketches of $A B$ and then report the median of the $d$ estimates.
If $d = \Omega(\log(1/\delta))$,  $m^2 = \Omega(\frac{\norm{C}_F^2}{\epsilon^2})$, then with probability $\geq 1- \delta$ there is $|\hat{C}_{ij}-C_{ij}| \leq \epsilon$. 
\end{theorem}

\begin{proof}
We have shown in Section~\ref{sec:tce} that $\HCS(AB) = \HCS(A)\HCS(B)$. Given Theorem~\ref{thm:HCS-vec}, we have $\E(\HCS(AB)) = AB$, $\Var(\hat{AB}_{ij}) \leq \norm{AB}_F^2/m^2 = \norm{C}_F^2/m^2$. From Chebychev’s inequality, if we run this sketch $d$ times, where $d = \Omega(\log(1/\delta))$, we can get the desired error bond with probability at least $1-\delta$.
\end{proof}

\end{appendices}

\newpage

\end{document}